\documentclass[12pt]{alt2020} 

\usepackage{bm}
\newcommand{\qedhere}{}
\usepackage{times}  
\usepackage{amsmath,mathtools}
\usepackage{dsfont,amssymb}
\usepackage{algorithm}
\usepackage{tikz}
\usepackage{url}  
\usepackage{verbatim}
\usepackage{color}
\usepackage{graphicx}  
\usepackage{algorithm}
\usepackage{algorithmic}
\usepackage{amssymb}
\usepackage{natbib}
\usepackage{multirow}
\usepackage{enumitem}
\usepackage{nicefrac}

\newtheorem{claim}{Claim}

\newcommand{\beq}{\begin{eqnarray}}
\newcommand{\eeq}{\end{eqnarray}}

\newfont{\bbb}{msbm10 scaled 500}

\newfont{\bb}{msbm10 scaled 1100}
\newcommand{\CC}{\mbox{\bb C}}
\newcommand{\RR}{\mbox{\bb R}}
\newcommand{\ZZ}{\mbox{\bb Z}}

\newcommand{\EE}{\mbox{\bb E}}

\newcommand{\Ac}{{\cal A}}
\newcommand{\Bc}{{\cal B}}

\newcommand{\Mc}{{\cal M}}
\newcommand{\Nc}{{\cal N}}

\newcommand{\remove}[1]{}

\newcommand{\avg}{{\mathbb E}}

\newcommand\reals{{\mathbb R}}

\newcommand\integers{{\mathbb Z}}

\newcommand{\variation}[1]{\left\|#1\right\|_{\textrm{TV}}}

\usepackage{comment}
\newcount\Comments  
\definecolor{darkgreen}{rgb}{0,0.5,0}
\definecolor{darkred}{rgb}{0.7,0,0}
\definecolor{purple}{rgb}{0.7,0,0.7}
\definecolor{teal}{rgb}{0.3,0.8,0.8}
\newcommand{\kibitz}[2]{\ifnum\Comments=1\textcolor{#1}{#2}\fi}

\newcommand{\f}[1]{\boldsymbol{#1}}
\newcommand{\ca}[1]{\mathcal{#1}}

\title[Algebraic and Analytic Approaches for Parameter Learning in Mixture Models]{Algebraic and Analytic Approaches for Parameter Learning in Mixture Models}
\usepackage{times}

\altauthor{%
 \Name{Akshay Krishnamurthy} \Email{akshay@cs.umass.edu}\\
 \addr Microsoft Research, New York City, NY 10011
 \AND
 \Name{Arya Mazumdar} \Email{arya@cs.umass.edu}\\
 \addr University of Massachusetts Amherst, MA 01003, USA%
 \AND
 \Name{Andrew McGregor} \Email{mcgregor@cs.umass.edu}\\
 \addr University of Massachusetts Amherst, MA 01003, USA%
 \AND
 \Name{Soumyabrata Pal} \Email{spal@cs.umass.edu}\\
 \addr University of Massachusetts Amherst, MA 01003, USA%
}

\begin{document}

\maketitle

\begin{abstract}%
 We present two different approaches for parameter learning in several mixture models in one dimension. Our first approach uses complex-analytic methods and applies to Gaussian mixtures with shared variance, binomial mixtures with shared success probability, and Poisson mixtures, among others. An example result is that $\exp(O(N^{1/3}))$ samples suffice to exactly learn a mixture of $k<N$ Poisson distributions, each with integral rate parameters bounded by $N$. Our second approach uses algebraic and combinatorial tools and applies to binomial mixtures with shared trial parameter $N$ and
differing success parameters, as well as to mixtures of geometric distributions. Again, as an example, for binomial mixtures with $k$ components and success parameters discretized to resolution $\epsilon$, $O(k^2(\nicefrac{N}{\epsilon})^{\nicefrac{8}{\sqrt{\epsilon}}})$ samples suffice to exactly recover the parameters. For some of these distributions, our results represent the first guarantees for parameter estimation.
\end{abstract}


\section{Introduction}
Mixture modeling is a powerful method in the statistical toolkit, with
widespread use across the
sciences~\citep{titterington1985statistical}. Starting with the
seminal work of~\cite{dasgupta1999learning}, computational and
statistical aspects of learning mixture models have been the subject
of intense investigation in the theoretical computer science and
statistics
communities~\citep{achlioptas2005spectral,kalai2010efficiently,belkin2010polynomial,arora2001learning,moitra2010settling,feldman2008learning,chan2014efficient,acharya2017sample,hopkins2018mixture,diakonikolas2018list,kothari2018robust,hardt2015tight}.

In this literature, there are two flavors of result: (1)
\emph{parameter estimation}, where the goal is to identify the mixing
weights and the parameters of each component from samples, and (2)
\emph{density estimation} or PAC-learning, where the goal is simply to
find a distribution that is close in some distance (e.g., TV distance)
to the data-generating mechanism. Density estimation can be further
subdivided into \emph{proper} and \emph{improper learning} approaches
depending on whether the algorithm outputs a distribution from the
given mixture family or not. These three guarantees are quite
different. Apart from Gaussian mixtures, where all types of results
exist, prior work for other mixture families largely focuses on
density estimation, and very little is known for parameter estimation
outside of Gaussian mixture models.  In this paper, we focus on
parameter estimation and provide two new approaches, both of which
apply to several mixture families.

Our first approach is analytic in nature and yields new sample
complexity guarantees for univariate mixture models including
Gaussian, Binomial, and Poisson. Our key technical insight is that we
can relate the total variation between two candidate mixtures to a
certain Littlewood polynomial, and then use complex analytic
techniques to establish separation in TV-distance. With this
separation result, we can use density estimation techniques
(specifically proper learning techniques) to find a candidate mixture
that is close in TV-distance to the data generating mechanism.  The
results we obtain via this approach are labeled as ``analytic" in
Table~\ref{tab:tab1}. This approach has recently led to important
advances in the trace reconstruction and population recovery problems;
see work by \cite{DeOS17}, \cite{NazarovP17}, and \cite{DeOS17x}.



Our second approach is based on the \emph{method of moments}, a
popular approach for learning Gaussian mixtures, and is more
algebraic.  Roughly, these algorithms are based on expressing moments
of the mixture model as polynomials of the component parameters, and
then solving a polynomial system using estimated moments. This approach
has been studied in some generality by~\citet{belkin2010polynomial}
who show that it can succeed for a large class of mixture
models. However, as their method uses non-constructive arguments from
algebraic geometry it cannot be used to bound how many moments are
required, which is essential in determining the sample complexity; see
a discussion in \cite[Section 7.6]{moitra2018algorithmic}. In
contrast, our approach does yield bounds on how many moments suffice
and can be seen as a quantified version of the results
in~\citet{belkin2010polynomial}. The results we obtain via this
approach are labeled as ``algebraic'' in Table~\ref{tab:tab1}.

The literature on mixture models is quite large, and we have just referred to a sample of most relevant papers here. A bigger overview on learning distributions can be found in the recent monographs such as \cite{moitra2018algorithmic,diakonikolas2016learning}.


\subsection{Overview of results}
As mentioned, an overview of our sample complexity results are
displayed in Table~\ref{tab:tab1}, where in all cases we consider a
uniform mixture of $k$ distributions. Our guarantees are for
\emph{exact} parameter estimation, under the assumption that the
mixture parameters are discretized to a particular resolution, given
in the third column of the table. Theorem statements are given in the
sequel.

At first glance the guarantees seem weak, since they all involve
exponential dependence in problem parameters. However, except for the
Gaussian case, these results are the first guarantees for parameter
estimation for these distributions. All prior results we are aware of
consider density
estimation~\citep{chan2013learning,feldman2008learning}.

For the mixtures of discrete distributions, such as binomial and negative binomial with shared trial parameter, or Poisson/geometric/chi-squared mixtures with certain discretizations, it seems like the dependence of sample complexity on the number of components $k$ is polynomial (see Table~\ref{tab:tab1}). Note that for these examples $k \le N$, the upper bounds on parameter values. Therefore the actual dependence on $k$ can still be interpreted as exponential. The results are especially interesting when $k$ is large and possibly growing with $N$.

For Gaussian mixtures, the most interesting aspect of our bound is the
polynomial dependence on the number of components $k$ (first row of
Table~\ref{tab:tab1}).  In our setting and taking $\sigma = 1$, the
result of~\citet{moitra2010settling} is applicable, and it yields
$\epsilon^{-O(k)}$ sample complexity, which is incomparable to our
$k^3\exp(O(\epsilon^{-2/3}))$ bound. Note that our result avoids an
exponential dependence in $k$, trading this off for an exponential
dependence on the discretization/accuracy parameter
$\epsilon$.\footnote{Due to our discretization structure, our results
  do not contradict the lower bounds
  of~\citet{moitra2010settling,hardt2015tight}.} Other results for
Gaussian mixtures either 1) consider density
estimation~\citep{daskalakis2014faster,feldman2008learning}, which is
qualitatively quite different from parameter estimation, 2) treat $k$
as constant~\citep{hardt2015tight,kalai2010efficiently}, or 3) focus
on the high dimensional setting and require separation assumptions
(see for example~\cite{diakonikolas2017statistical}
and~\cite{moitra2018algorithmic}).

As such, our results
  reflect a new sample complexity tradeoff for parameter estimation
in Gaussian mixtures.



As another note, using ideas from~\citep{NazarovP17,DeOS17}, one can show that the analytic result for Binomial mixtures is optimal. 
This
raises the question of whether the other results are also optimal or
is learning a Binomial mixture intrinsically harder than learning,
e.g., a Poisson or Gaussian mixture?

As a final remark, our assumption that parameters are discretized is
related to separation conditions that appear in the literature on
learning Gaussian mixtures. However, our approach does not seem to
yield guarantees when the parameters do not exactly fall into the
discretization. We hope to resolve this shortcoming in future work.

\begin{table}
\begin{center}
\renewcommand{\arraystretch}{2}
\begin{tabular}{|c | c | c | c | c |}
\hline
Distribution & Pdf/Pmf $f(x;\theta)$ & Discretization & Sample Complexity & Approach\\
\hline\hline
Gaussian & $\frac{1}{\sqrt{2\pi}\sigma}e^{-\frac{(x-\mu)^{2}}{\sigma^2}}$ & $\mu_i \in \epsilon\ZZ$ & $k^3 \exp(O((\sigma/\epsilon)^{2/3}))$ & Analytic\\
\hline
\multirow{ 2}{*}{Binomial} & \multirow{ 2}{*}{${n \choose x}p^{x}(1-p)^{n-x}$} & $n_i \in \{1,2,\ldots,N\}$ & $\exp(O(((N/p)^{1/3})))$ & Analytic \protect\footnotemark\\
 \cline{3-5} 
& & $p_i \in  \{0,\epsilon,\ldots,1\}$ & $O(k^2(n/\epsilon)^{8/\sqrt{\epsilon}})$ & Algebraic \\
\hline
Poisson & $\frac{\lambda^x e^{-\lambda}}{x!}$ & $\lambda_i \in \{0,1,\dots,N\}$ & $\exp(O(N^{1/3})) $ & Analytic\\ \hline
\multirow{ 2}{*}{Geometric} & \multirow{ 2}{*}{$(1-p)^x p $} & $1/p_i \in \{1,\dots,N\}$ & $O(k^2(\sqrt{N})^{8\sqrt{N}})$  & Algebraic\\  \cline{3-5} 
& & $p_i \in \{0,\epsilon,\dots,1\}$ & $O(\frac{k^2}{\epsilon^{8/\sqrt{\epsilon}+2}}\log \frac{1}{{\epsilon}})$  & Algebraic\\ \hline
$\chi^2$ & $\frac{x^{n/2-1}e^{-x/2}}{2^{n/2}\Gamma(n/2)}$ & $n_i \in \{0,1, \dots,N\}$ &  $\exp(O(N^{1/3})) $ & Analytic \\ \hline
Negative Binomial & ${x+r-1 \choose x} (1-p)^rp^x$ & $r_i \in  \{1,2,\ldots, N\}$ & $\exp(O((N/p)^{1/3})) $ & Analytic\\ \hline
\end{tabular}
\end{center}
\caption{Overview of our results. Results are given for uniform mixtures of $k$ different components but some can be extended to non-uniform mixtures.\label{tab:tab1} Note that for rows 2, 4, 7, and 8, $k$ does not appear. This is because $k\leq N$ and other terms dominate.}
\end{table}

\subsection{Our techniques}
To establish these results, we take two loosely related approaches.
In our analytic approach, the key structural result is to lower bound
the total variation distance between two mixtures $\Mc,\Mc'$ by a
certain Littlewood polynomial. For each distribution type, if the
parameter is $\theta$, we find a function $G_t: \RR \to \CC$ such that
\begin{align*}
\EE[G_t(X)] = \exp(it\theta).
\end{align*}
(For Gaussians, $G_t$ is essentially the characteristic function).
Such functions can be used to obtain Littlewood polynomials from the
difference in expectation for two different mixtures, for example if
the parameters $\theta$ are integral and the mixture weights are
uniform. Applying complex analytic results on Littlewood polynomials,
this characterization yields a lower bound on the total variation
distance between mixtures, at which point we may use density
estimation techniques for parameter learning. Specifically we use the
minimum distance estimator (see,~\citet[Sec.~6.8]{devroye2012combinatorial}), which
is based on the idea of Scheffe sets. Scheffe sets are building blocks of the Scheffe estimator, 
commonly used in density estimation, e.g.~\citet{suresh2014near}. 


\footnotetext{We obtained this result as a byproduct of sparse trace reconstruction~\citep{krishnamurthy2019trace}. In fact, the present work was motivated by the observation that the  technique we were using there is much more general.}

  



Our algebraic approach is based on the more classical method of
moments. Our key innovation here is a combinatorial argument to bound
the number of moments that we need to estimate in order to exactly
identify the correct mixture parameters. In more detail, when the
parameters belong to a discrete set, we show that the moments reveal
various statistics about the multi-set of parameters in the
mixture. Then, we adapt and extend classical combinatorics results on
sequence reconstruction to argue that two distinct multi-sets must
disagree on a low-order moment. These combinatorial results are
related to the Prouhet-Tarry-Escott problem (see, e.g.,
\cite{Borwein2002}) which also has connections to Littlewood
polynomials. To wrap up we use standard concentration arguments to
estimate all the necessary moments, which yields the sample complexity
guarantees.

We note that the complex analytic technique provides non-trivial result only for those mixtures for which an appropriate function $G_t$ exists. On the other hand, the algebraic approach works for all mixtures whose $\ell^{th}$ moment can be described as a polynomial of degree exactly $\ell$ in its unknown parameters. In \cite{belkin2010polynomial}, it was shown that most distributions have this later property. In general, where both methods can be applied, the complex analytic techniques typically provide tighter sample complexity bounds than the algebraic ones.



\section{Learning Mixtures via Characteristic Functions}
In this section, we show how analysis of the characteristic function
can yield sample complexity guarantees for learning mixtures. At a
high level, the recipe we adopt is the following.
\begin{enumerate}
\item First, we show that, in a general sense, the total variation distance between two
  separated mixtures is lower bounded by the $L_{\infty}$ norm of
  their characteristic functions.
\item Next, we use complex analytic methods and specialized arguments
  for each particular distribution to lower bound the latter norm.
\item Finally, we use the minimum distance
  estimator~\citep{devroye2012combinatorial} to find a mixture that is
  close in total variation to the data generating distribution. Using
  uniform convergence arguments this yields exact parameter learning.
\end{enumerate}

The two main results we prove in this section are listed below.
\begin{theorem}[Learning Gaussian mixtures]{\label{thm:main}}
Let $\mathcal{M}=\frac{1}{k}\sum_{i=1}^k\mathcal{N}(\mu_i,\sigma^2)$
be a uniform mixture of $k$ univariate Gaussians, with known shared
covariance $\sigma^2$ and with distinct means $\mu_i \in
\epsilon\mathbb{Z}$. Then there exists an algorithm that requires $k^3
\exp(O((\sigma/\epsilon)^{2/3}))$ samples from $\mathcal{M}$ and
exactly identifies the parameters $\{\mu_i\}_{i=1}^k$ with high
probability.
\end{theorem}

\begin{theorem}[Learning Poisson mixtures]
Let $\mathcal{M} = \frac{1}{k}\sum_{i=1}^k\textrm{Poi}(\lambda_i)$
where $\lambda_i \in \{0,1,\ldots,N\}$ for each $i$ are distinct. Then there exists
an algorithm that that requires $\exp(O(N^{1/3})))$ samples 
from $\mathcal{M}$ to exactly identify the parameters
$\{\lambda_i\}_{i=1}^k$ with high probability.
\end{theorem}

There are some technical differences in deriving the results for
Gaussian vs Poisson mixtures.  Namely, because of finite choice of
parameters we can take a union bound over the all possible incorrect
mixtures for the latter case, which is not possible for Gaussian. For
Gaussian mixtures we instead use an approach based on VC dimension.
The results for negative binomial mixtures and chi-squared mixtures
(shown in Table~\ref{tab:tab1}) follow the same route as the Poisson
mixture.  As reported in Table~\ref{tab:tab1}, this approach also
yields results for mixtures of binomial distributions that we obtained in a different context in our prior work~\citep{krishnamurthy2019trace}.


%


\subsection{Total Variation and Characteristic Functions}
Let $\{f_\theta\}_{\theta \in \Theta}$ denote a parameterized family
of distributions over a sample space $\Omega \subset \mathbb{R}$,
where $f_\theta$ denotes either a pdf or pmf, depending on the
context.  We call $\Mc$ a (finite) \emph{$\Theta$-mixture} if $\Mc$
has pdf/pmf $\sum_{\theta \in \Ac}\alpha_\theta f_\theta$ and $\Ac
\subset \Theta, |\Ac| =k$. For a distribution with density $f$ (we use
distribution and density interchangeably in the sequel), define the
\emph{characteristic function} $C_f(t) \equiv \mathbb{E}_{X \sim
  f}[e^{itX}]$. For any two distribution $f,f'$ defined over a sample
space $\Omega \subseteq \mathbb{R}$ the variational distance (or the
TV-distance) is defined to be
$\variation{f - f'} \equiv \frac{1}{2}\int_{\Omega}
\left|\frac{df'}{df} - 1\right| df$.
For a function $G:\Omega \to \CC$ define the $L_{\infty}$ norm to be
$\|G\|_{\infty} = \sup_{\omega \in \Omega} |G(\omega)|$ where
$|\cdot|$ denotes the modulus.

As a first step, our aim is to show that the total variation distance
between $\Mc=\sum_{\theta \in \Ac}\alpha_\theta f_\theta$ and any
other mixture $\Mc'$ given by $\sum_{\theta \in \Bc}\beta_\theta
f_\theta, \Bc \subset \Theta, |\Bc| =k$ is lower bounded. The
following elementary lemma completes the first step of the outlined
approach.

\begin{lemma}\label{lem:chartv}
For any two distributions $f,f'$ defined over the same sample space $\Omega\subseteq\mathbb{R}$, we have
\begin{align*}
\variation{f -f'} \ge \frac{1}{2} \sup_{t \in \reals}|C_f(t) -C_{f'}(t)|.
\end{align*}
More generally, for any $G: \Omega \to \CC$ and $\Omega' \subset \Omega$ we have
\begin{align*}
\variation{f -f'} \geq \left(2\sup_{x \in \Omega'}|G(x)|\right)^{-1} \Big(\left| \EE_{X \sim f} G(X) - \EE_{X \sim f'}G(X) \right| \\\qquad \,\,- \int_{x \in\Omega \setminus \Omega'}|G(x)| \cdot |df(x) - df'(x)|\Big).
\end{align*}
\end{lemma}
\begin{proof}
We prove the latter statement, which implies the former since for the
function $G(x) = e^{itx}$ we have $\sup_{x} |G(x)| = 1$. 
We have
\begin{align*}
|\EE_{X \sim f} G(X) &- \EE_{X \sim f'}G(X)| \leq \int_{x \in \Omega}|G(x)| \cdot |df(x) - df'(x)| \\
& \le 2\sup_{x \in \Omega'}|G(x)| \cdot \variation{f -f'} + \int_{x \in\Omega \setminus \Omega'}|G(x)| \cdot |df(x) - df'(x)|. \tag*\qedhere
\end{align*}
\end{proof}

Equipped with the lower bound in Lemma~\ref{lem:chartv}, for each type
of distribution, we set out to find a good function $G$ to witness
separation in total variation distance. As we will see shortly, for a
parametric family $f_\theta$, it will be convenient to find a family
of functions $G_t$ such that
\begin{align*}
\EE_{X \sim f_\theta}[G_t(X)] = \exp(it\theta).
\end{align*}
Of course, to apply Lemma~\ref{lem:chartv}, it will also be important
to understand $\|G_t\|_{\infty}$. While such functions are specific to
the parametric model in question, the remaining analysis will be
unified. We derive such functions and collect the relevant properties
in the following lemma. At a high level, the calculations are based on
reverse engineering from the characteristic function, e.g., finding a
choice $t'(t)$ such that $C_f(t') = \exp(it\theta)$.

\begin{lemma}\label{lem:list}
Let $z=\exp(i t)$ where $t \in [-\pi/L,\pi /L]$. 
\begin{itemize}
\item { Gaussian}. If $X\sim \Nc(\mu,\sigma)$ and $G_t(x) = e^{itx}$ then 
\[\EE[G_t(X)]=\exp(-\sigma^2 t^2/2)z^\mu \mbox{ and } \|G_t\|_{\infty}= 1 \ . \]
\item { Poisson}. If $X\sim \textrm{Poi}(\lambda)$ and $G_t(x)=(1+it)^x$ then 
\[\EE[G_t(X)]=z^\lambda \mbox{ and } |G_t(x)| \leq (1+t^2)^{x/2} \ . \]
\item { Chi-Squared}. If $X\sim \chi^2(\ell)$ and $G_t(x)=\exp(x/2-xe^{-2it}/2)$  then 
\[ \EE[G_t(X)]=z^\ell \mbox{ and } |G_t(x)|\leq e^{cxt^{2}+O(xt^{4})} \ .\]
\item { Negative Binomial}. If $X\sim \textrm{NB}(r,p)$ and $G_t(x)=\left({1}/{p}-({1}/{p}-1)e^{-it}\right)^x$ then 
\[\EE[G_t(X)]=z^r \mbox{ and } |G_t(x)|\leq e^{-cx\frac{(1-p)t^2}{p^2}} \ .\]
\end{itemize}
\end{lemma}
\begin{proof}
Here we give the argument for Poisson distributions only. The
remaining calculations are deferred to the appendix.  For Poisson
random variables, if $G_t(x) = (1+it)^x$ then since $|1+it|^2 = 1 +
t^2$ the second claim follows. For the first:
\begin{align*}
\EE[G_t(X)]=\exp(\lambda((1+it)-1))=z^\lambda.\tag*\qedhere
\end{align*}
\end{proof}


\subsection{Variational Distance Between Mixtures}
We crucially use the following lemma. 
\begin{lemma}[\cite{BorweinE97}] \label{lem:npb} Let $a_0, a_1, a_2, \dots \in \{0,1,-1\}$ be such that not all of them are zero. For any complex number $z$, let $A(z) \equiv \sum_k a_k z^k.$ Then, for some absolute constant $c$, 
$$
\max_{-\pi/L \le t \le \pi/L} |A(e^{it})| \ge e^{-cL} \ . 
$$
\end{lemma}

We will also need the following `tail bound' lemma.
\begin{lemma}\label{lem:tail}
Suppose $a>1$ is any real number  and $r \in \reals_+$. 
For any discrete random variable $X$ with support $\integers$ and pmf $f$,
$$
\sum_{x\ge r} a^x f(x) \le \frac{\EE[a^{2X}]}{a^{r-1}}.
$$
\end{lemma}
\begin{proof}
Note that, $\Pr(X \ge x) =  \Pr(a^{2X-2x} \ge 1) \le \EE[a^{2X-2x}]$. We have,
\begin{align*}
\sum_{x\ge r} a^x \Pr(X =x) &\le \sum_{x\ge r} a^x \Pr(X \ge x)
\le  \sum_{x\ge r} a^x \EE[a^{2X-2x}] =  \EE[a^{2X}] \sum_{x\ge r} a^{-x} \le \frac{\EE[a^{2X}]}{a^{r-1}}.\tag*\qedhere
\end{align*}
\end{proof}

\begin{theorem}[TV Lower Bounds] \label{thm:tv} The following bounds hold on distance between two different mixtures assuming all $k$ parameters are distinct for each mixture.
\begin{itemize}[leftmargin=10pt]
\item \emph{Gaussian:} $\mathcal{M} =
  \frac{1}{k}\sum_{i=1}^k \Nc(\mu_i,\sigma)$
and 
$\mathcal{M}' =
  \frac{1}{k}\sum_{i=1}^k \Nc(\mu_i',\sigma)$
where 
$\mu_i, \mu_i' \in \epsilon \ZZ$.
Then
\[
\variation{\mathcal{M}' -\mathcal{M}} \geq k^{-1} \exp(-\Omega((\sigma/\epsilon)^{2/3})) \ .
\]
\item \emph{Poisson:} $\mathcal{M} =
  \frac{1}{k}\sum_{i=1}^k \textrm{Poi}(\lambda_i)$
and 
$\mathcal{M}' =
  \frac{1}{k}\sum_{i=1}^k \textrm{Poi}(\lambda_i')$
where 
$\lambda_i, \lambda_i' \in \{0, 1,\ldots, N\}$.
Then \[\variation{\mathcal{M}' -\mathcal{M}} \geq k^{-1} \exp(-\Omega(N^{1/3})) \ .\]
\item \emph{Chi-Squared:}  $\mathcal{M} =
  \frac{1}{k}\sum_{i=1}^k \chi^2(\ell_i)$
and 
$\mathcal{M}' =
  \frac{1}{k}\sum_{i=1}^k \chi^2(\ell_i')$
where  $\ell_i, \ell_i' \in \{ 1, 2, \ldots, N\}$. 
Then \[ \variation{\mathcal{M}' -\mathcal{M}} \geq k^{-1} \exp(- \Omega(N^{1/3})) \ .\] 
\item \emph{Negative Binomial:} 
$\mathcal{M} =
  \frac{1}{k}\sum_{i=1}^k \textrm{NB}(r_i,p)$
and 
$\mathcal{M}' =
  \frac{1}{k}\sum_{i=1}^k \textrm{NB}(r_i',p)$
where $r_i, r_i' \in \{1, 2, \ldots , N \}$. Then
\[\variation{\mathcal{M}' -\mathcal{M}} \geq  k^{-1} \exp(- \Omega((N/p)^{1/3})) \ .\]

\end{itemize}
\end{theorem}
\begin{proof}
As above we give the argument for Poisson random variables, deferring the others to the appendix. 
Let $X \sim \Mc$ and $X' \sim \Mc'$. Then, for $w=1+it$, from Lemma~\ref{lem:list},
$$
\EE(w^X) - \EE(w^{X'}) = \frac1k\sum_{j=1}^k (e^{it \lambda_j} - e^{it \lambda'_j}) .
$$

Now we use  
 Lemma~\ref{lem:chartv} with $G(x) = w^x$, $\Omega' = \{0,1, \dots, 2N\}$ and $t \le 1$, to have, 
\begin{align*}
| \EE(w^X) - \EE(w^{X'})| &= \Big|\sum_x (w^x \Mc(x) - w^x \Mc'(x)) \Big|
 \le \sum_x |w|^x  |\Mc(x) - \Mc'(x)| \\
& \le  (1+t^2)^{2N} \sum_x  |\Mc(x) - \Mc'(x)| + \sum_{x> 4N} (1+t^2)^{x/2}e^{-N}\frac{N^x}{x!}\\
&\le (1+t^2)^{2N} \sum_x  |\Mc(x) - \Mc'(x)| + \sum_{x> 4N} 2^{x/2}e^{-N}\frac{N^x}{x!}.
 \end{align*}
Now using Lemma \ref{lem:tail}, 
\begin{align*}
| \EE(w^X) - \EE(w^{X'})|& \le  2(1+t^2)^{2N} \variation{\Mc -\Mc'} + \frac{\EE[2^X]}{2^{2N-1/2}}
\le 2e^{2t^2N} \variation{\Mc -\Mc'} + \frac{e^{N} }{2^{2N-1/2}}\\
& =2e^{2\pi^2N/L^2} \variation{\Mc -\Mc'} + \exp(-\Omega(N)),
 \end{align*}
 by taking  $|t| \le \frac{\pi}{L}$. 
Now using Lemma~\ref{lem:npb}, there exist an absolute constant $c$ such that,
\begin{align*}
\max_{-\frac{\pi}{ L}\le t \le \frac{\pi}{ L}} \big| \sum_{j=1}^k (e^{it \lambda_j} - e^{it \lambda'_j})\big| \ge e^{-cL}.
\end{align*}
Therefore by setting $L = N^{1/3}$,
\begin{align*}
 \variation{\Mc -\Mc'} \ge  (2k)^{-1} e^{-cL-2\pi^2N /L^2} - \exp(-\Omega(N)) \ge  k^{-1}\exp(- \Omega(N^{1/3})).\tag*\qedhere
\end{align*}

\end{proof}

 \subsection{Parameter Learning}\label{sec:imp}
%

 \paragraph{Union Bound Approach for Discrete Distributions}
 We begin with the following proposition which follows from Theorem 7.1 of~\citet{devroye2012combinatorial}.
 \begin{lemma}\label{lem:uni}
 Suppose $F= \{f_\nu\}_{\nu \in \Theta}$ is a class of distribution such that for any $\nu,\nu' \in \Theta$, $\variation{f_\nu-f_{\nu'}} \ge \delta$. Then  $O(\log |\Theta|/\delta^2)$ samples from a distribution $f$ in $F$ suffice to distinguish it from all other distributions in $F$ with high probability.
 \end{lemma}

For the mixture of Poissons, $\mathcal{M} =
\frac{1}{k}\sum_{i=1}^k \textrm{Poi}(\lambda_i)$ where $\lambda_i \in
\{0, 1,\ldots, N\},$ the number of choices for parameters in the
mixture is $(N+1)^k$. Now using Lemmas~\ref{thm:tv} and \ref{lem:uni},
$\exp(O(N^{1/3}))$ samples are sufficient to learn the
parameters of the mixture.

Exactly the same argument applies to
mixtures of Chi-Squared and Negative-Binomial
distributions, yielding $\exp(O(N^{1/3}))$ and $
\exp(O((N/p)^{1/3}))$ samples suffice, respectively. However, for Gaussians we need a more intricate approach. 


 \paragraph{VC Approach for Gaussians} To learn the parameters of a
 Gaussian mixture
 \[\mathcal{M} =
  \frac{1}{k}\sum_{i=1}^k \Nc(\mu_i,\sigma) ~~\mbox{ where }~~\mu_i \in \{\ldots, -2\epsilon, -\epsilon, 0,\epsilon, 2\epsilon \ldots \} \]
  we use the minimum distance estimator precisely defined in \cite[Section~6.8]{devroye2012combinatorial}. Let $\Ac \equiv \{\{x: \Mc(x) \ge \Mc'(x)\}: \text{ for any two mixtures } \Mc \ne \Mc'\}$ be a collection of subsets. Let $P_m$ denote the empirical probability measure induced by the $m$ samples. Then, choose a mixture $\hat{\Mc}$ for which 
  the quantity $\sup_{A\in \Ac}  |\Pr_{\sim\hat{\Mc}}(A) - P_m(A) |$ is minimum (or within $1/m$ of the infimum). This is the minimum distance estimator, whose performance is guaranteed by the following proposition~\cite[Thm.~6.4]{devroye2012combinatorial}.
  

  
  \begin{proposition}
    Given $m$ samples from $\Mc$ and with $\Delta = \sup_{A \in \Ac}|\Pr_{\sim\Mc}(A) - P_m(A)|$, we have
  $$
  \variation{\hat{\Mc} -\Mc} \le 4\Delta +\frac{3}{m}.
  $$
\end{proposition}
We now upper bound the right-hand side of the above inequality.  Via
McDiarmid's inequality and a standard symmetrization argument, $\Delta$ is concentrated around its mean which is a function of $VC(\Ac)$, the VC dimension of the class $\Ac$, see
\cite[Section~4.3]{devroye2012combinatorial}:  
$$
\variation{\hat{\Mc} - \Mc} \leq 4\Delta + O(1/m) \leq 4\avg_{\sim\Mc} \Delta + O(1/\sqrt{m}) \le c \sqrt{\frac{VC(\Ac)}{m}},
$$ with high probability, for an absolute constant $c$.  This latter term is bounded by the following. 

 
\begin{lemma}
For the class $\Ac$ defined above, the VC dimension is given by $VC(\Ac) = O(k)$.
\end{lemma}
\begin{proof}
First of all we show that any element of the set $\Ac$ can be written
as union of at most $4k-1$ intervals in $\reals$. For this we use the
fact that a linear combination of $k$ Gaussian pdfs $f(x) =
\sum_{i=1}^{k} \alpha_{i} f_{i}(x)$ where $f_i$s normal pdf
$\Nc(\mu_i,\sigma^2_i)$ and $\alpha_i \in \reals, 1\le i\le k$ has at
most $2k-2$ zero-crossings \citep{kalai2012disentangling}. Therefore,
for any two mixtures of interest $\Mc(x) -\Mc'(x)$ has at most $4k-2$
zero-crossings. Therefore any $A\in \Ac$ must be a union of at most
$4k-1$ contiguous regions in $\reals$.  It is now an easy exercise to
see that the VC dimension of such a class is $\Theta(k)$.
\end{proof}
%
%


As a result the error of the minimum distance estimator is $O(\sqrt{k/m})$ with high probability. 
But from Theorem~\ref{thm:tv}, notice that for any other mixture $\Mc'$ we must have,
$$
\variation{\Mc -\Mc'} \ge  k^{-1} \exp(-\Omega((\sigma/\epsilon)^{2/3})).
$$ 
As long as $\variation{\hat{\Mc} -\Mc} \le \frac12 \variation{\Mc -\Mc'}$ we will exactly identify the parameters. Therefore 
$m = k^3 \exp(O((\sigma/\epsilon)^{2/3}))$ samples suffice to exactly learn the parameters with high probability.



\subsection{Extension to Non-Uniform Mixtures}
The above results extend to non-uniform mixtures, where the main change is that we require a generalization of Lemma~\ref{lem:npb}. The result, also proved by \cite{BorweinE97}, states that if $a_0, a_1, a_2, \ldots \in [-1, 1]$ with $\textrm{poly}(n)$ precision then $\max_{-\pi/L \le \theta \le \pi/L} |A(e^{i\theta})| \ge e^{-cL \log n}$, for an absolute constant $c$. This weaker bound yields an extra $\textup{poly}(n)$ factor in the sample complexity.  

\section{Learning Mixtures via Moments}

There are some mixtures where the problem of learning parameters is not amenable to the approach in the previous section. A simple motivating example is learning the parameters $p_i\in \{0, \epsilon, 2\epsilon, 3\epsilon, \ldots, 1\}$ values\footnote{Note that we are implicitly assuming $1/\epsilon$ is integral here and henceforth.} in the mixture $\mathcal{M}=\frac{1}{k} \sum_{i=1}^k \textup{Bin}(n,p_i)$. In this section, we present an alternative procedure for learning such mixtures. The basic idea is as follows:
\begin{itemize}
\item We compute moments $\avg X^\ell$ exactly for $\ell=0,1,\ldots, T$ by taking sufficiently many samples. The number of samples will depend on $T$ and the precision of the parameters of the mixture.
\item We argue that if $T$ is sufficiently large, then these moments uniquely define the parameters of the mixture. To do this we use a combinatorial result due to \cite{KRASIKOV1997344}.
\end{itemize}

In this section, it will be convenient to define a function $m_\ell$ on multi-sets where 
\[m_\ell (A):=\sum_{a\in A} a^\ell  \ .\]

Our main result is as follows:

\begin{theorem}[Learning Binomial mixtures]
Let $\mathcal{M}=\frac{1}{k}\sum_{i=1}^k\textrm{Bin}(n,p_i)$
be a uniform mixture of $k$ binomials, with known shared
number of trials $n$ and unknown probabilities $p_1, \ldots, p_k \in \{0,\epsilon,2\epsilon, \ldots, 1\}$. Then, provided $n \ge 4/\sqrt{\epsilon}$, the first $4/\sqrt{\epsilon}$ moments suffice to learn the parameters $p_i$ and  there exists an algorithm that, when given
$O(k^2 (n/\epsilon)^{8/\sqrt{\epsilon}})$ samples from
$\mathcal{M}$, exactly identifies the parameters $\{p_i\}_{i=1}^k$
with high probability.
\end{theorem}

\paragraph{Computing the Moments}
We compute the $\ell$th moment as $S_{\ell,t}=\sum Y_i^\ell /t$ where $Y_1, \ldots, Y_t\sim X$. 
\begin{lemma}\label{lem:boundmoments}
$\Pr[|S_{\ell,t}-\avg X^\ell|\geq \gamma] \leq \frac{\avg X^{2\ell}}{t\gamma^2} \leq \frac{(2\ell)! } {\gamma^2 t} \inf_\alpha \left ( \frac{\avg e^{\alpha X}}{\alpha^{2\ell}} \right )$ where the last inequality assumes the all the moments of $X$ are non-negative. 
\end{lemma}
\begin{proof}
By the Chebyshev bound, 
\[\Pr[|S_{\ell,t}-\avg X^\ell|\geq \gamma] \leq \frac{Var(S_{\ell,t})}{\gamma^2}=\frac{Var(X^\ell)}{t\gamma^2}
\leq
\frac{\avg X^{2\ell}}{t\gamma^2}.
\]
We then use the moment generating function:
 for all $\alpha>0$, $\avg X^{2\ell}\leq (2\ell)! \avg e^{\alpha X}/\alpha^{2\ell}$.
\end{proof}

The following corollary, tailors the above lemma for a mixture of binomial distributions.
\begin{corollary}\label{bincorol}
If $X\sim \sum_{i=1}^k \textup{Bin}(n,p_i)/k$ then $\Pr[|S_{\ell,t}-\avg X^\ell |\geq \gamma]= \gamma^{-2}n^{2\ell}/t$. 
\end{corollary}

Fixing $n$, the $\ell^{\textrm{th}}$ moment of a mixture of binomial distributions $X\sim \sum_{i=1}^k \textup{Bin}(n,p_i)/k$ 
is
\[\avg X^\ell=\sum_{i=1}^k f(p_i)/k\] where $f$ is a polynomial of degree at most $\ell$ with integer coefficients~\citep{belkin2010polynomial}. If $p_i$ is an integer multiple of $\epsilon$ then this implies $k(\avg X^\ell)/\epsilon^\ell$ is integral and therefore any mixture with a different $\ell$th moment differs by at least $\epsilon^\ell/k$. Hence, learning the $\ell$th moment up to $\gamma_\ell<\epsilon^\ell/(2k)$ implies learning the moment exactly. 
\begin{lemma}{\label{lem:ind}}
For $X\sim \textup{Bin}(n,p)$, $\avg X^{\ell}$ is a polynomial in $p$ of degree exactly $\ell$ if $n \ge \ell$.
\end{lemma}
The proof of the lemma is relegated to the appendix.
\begin{theorem}
$O(k^2  (n/\epsilon)^{8/\sqrt{\epsilon}})$ samples are sufficient to exactly learn the first $4/\sqrt{\epsilon}$ moments of a uniform mixture of $k$ binomial distributions $\sum_{i=1}^k \textup{Bin}(n,p_i)/k$ with probability at least $7/8$ where each $p_i\in \{0,\epsilon, 2\epsilon, \ldots, 1\}$.
\end{theorem}
\begin{proof}
Let $T=4/\sqrt{\epsilon}$. From Corollary \ref{bincorol} and the preceding discussion, learning the $\ell $th moment exactly with failure probability $1/9^{1+T-\ell}$ requires 
\[
t= \gamma_\ell^{-2} n^{2\ell} 9^{1+T-\ell} = O(k^2 9^{1+T-\ell} n^{2\ell}/\epsilon^{2\ell})= O(k^2 9^{T} (n/3\epsilon)^{2\ell})
\]
samples. And hence, we can compute all $\ell$th moments exactly for $1\leq \ell\leq 4/\sqrt{\epsilon}$ using \[\sum_{\ell=1}^T O(k^2 9^{T} (n/3\epsilon)^{2\ell})=O(k^2 (n/\epsilon)^{2T})\] samples with failure probability $\sum_{\ell=1}^T 1/9^{1+T-\ell}<\sum_{i=1}^\infty 1/9^{i}=1/8$.
\end{proof}

\paragraph{How many moments determine the parameters}
It remains to show the first $4/\sqrt{\epsilon}$ moments suffice to determine the $p_i$ values 
in the mixture $X\sim \sum_{i=1}^k \textup{Bin}(n,p_i)/k$ provided $n \ge \frac{4}{\epsilon}$. To do this suppose there exists another mixture $Y\sim \sum_{i=1}^k \textup{Bin}(n,q_i)/k$ and we will argue that 
\[
\avg X^\ell = \avg Y^\ell \mbox{ for } \ell=0, 1, \ldots, 4\sqrt{1/\epsilon} 
\]
implies $\{p_i\}_{i\in [k]}=\{q_i\}_{i\in [k]}$. To argue this, define integers $\alpha_i,\beta_i \in \{0,1,\ldots, 1/\epsilon\}$ such at that $p_i=\alpha_i \epsilon$ and $q_i=\beta_i \epsilon$. Let $\mathcal A=\{\alpha_1, \ldots, \alpha_k\}$ and $\mathcal B=\{\beta_1, \ldots, \beta_k\}$ . Then,
\[
\avg X = \avg Y   \Longrightarrow   \sum_i \alpha_i=\sum_i \beta_i \Longrightarrow m_1(\mathcal A) = m_1(\mathcal B) \]
and, after some algebraic manipulation, it can be shown that for all $\ell\in \{2,3,\ldots\}$, 
\begin{eqnarray*}
\left (\forall \ell'\in \{0,1,\ldots, \ell-1\}~,~ \sum_i \alpha_i^{\ell'} =\sum_i \beta_i^{\ell'} \right )~\mbox{ and }~ \avg X^\ell = \avg Y^\ell \\
 \Longrightarrow  \left ( \sum_i \alpha_i^\ell=\sum_i \beta_i^\ell \right ) \Longrightarrow m_\ell (\mathcal A) = m_\ell(\mathcal B) \ .
\end{eqnarray*}
Hence, if the first $T$ moments match $m_\ell(\mathcal A)= m_\ell(\mathcal B)$ for all $\ell=0,1,\ldots, T$. But the following theorem establishes that if $T=4\sqrt{1/\epsilon}$ then this implies $\mathcal A=\mathcal B$.

\begin{theorem}[\cite{KRASIKOV1997344}]\label{kras}
For any two subsets $S,T$ of $\{0,1,\dots,n-1\}$, then 
\[S=T \mbox{ iff } \left (m_k(S) =  m_k(T)   \mbox{ for all }  k=0,1,\dots,4 \sqrt{n} \right )  \ . \]
\end{theorem}

We note that the above theorem is essentially tight. Specifically, there exists $S\neq T$ with $m_k(S) =  m_k(T) $ for $k=0,1,\ldots, cn/\log n$ for some $c$. As a consequence of this, we note that even the exact values of the $c\sqrt{n}/\log n$ moments are insufficient to learn the parameters of the distribution. For an example in terms of Gaussian mixtures, even given the promise $\mu_i \in \{0,1, \ldots, n-1\}$ are distinct, then the first  $c\sqrt{n}/\log n$ moments of $\sum_i \mathcal{N}(\mu_i,1)$ are insufficient to uniquely determine $\mu_i$ whereas the first $4\sqrt{n}$ moments are sufficient.


\subsection{Extension to Non-Uniform Distributions}
We now consider extending the framework to non-uniform distributions. In this case, the method of computing the moments is identical to the uniform case. However, when arguing that a small number of moments suffices we can no longer appeal to the Theorem \ref{kras}.

To handle non-uniform distribution we introduce a precision variable $q$ and assume that the weights of the component distributions $\omega_1, \omega_2, \ldots, \omega_k$ are of the form:
\[\omega_i=\frac{w_i}{\sum_{i=1}^k w_i}\]
where $w_i\in \{0,1,\ldots, q-1\}$. Then, in the above framework if we are trying to learn parameters $\alpha_1, \ldots, \alpha_k$ then the moments are going to define a multi-set consisting of $w_i$ copies of $\alpha_i$ for each $i\in [k]$. To quantify how many moments suffice in this case, we need to prove a variant of Theorem \ref{kras}. The proof is a relatively straight-forward generalization of proof by \cite{Scott97} and can be found in the appendix.

\begin{theorem}\label{thm:multiplicity}
For any two multi-sets $S,T$ where each element is in  $\{0,1,\dots,n-1\}$ and the multiplicity of each element is at most $q-1$, then $S=T$ if and only if  $m_k(S) =  m_k(T)   \mbox{ for all }  k=0,1,\dots,2 \sqrt{qn\log qn}$.
\end{theorem}

%
%


\paragraph{Acknowledgements}

The work was partially supported by NSF grants CCF-1909046, CCF-1934846, CCF-1908849, and CCF-1637536. 
{

}
\section{Omitted Proofs}

\paragraph{Additional calculations for Lemma \ref{lem:list}.}
We consider each distribution in turn:
\begin{itemize}
\item \emph{Gaussian:} Observe that $\EE[G_t(X)]$ is precisely the characteristic function. Clearly we have $\|G_t\|_{\infty} = 1$ and further
\[\EE[G_t(X)]=\exp(it \mu -\sigma^2 t^2/2)=\exp(-\sigma^2 t^2/2) z^\mu.\]

\item \emph{Poisson:} If $G_t(x) = (1+it)^x$ then since $|1+it|^2 = 1 + t^2$ the second claim follows. For the first:
\begin{align*}
\EE[G_t(X)]=\exp(\lambda((1+it)-1))=z^\lambda.\tag*\qedhere
\end{align*}

\item \emph{Chi-Squared:} Let $w_t=\exp(1/2-e^{-2it}/2)$ then $|w_t|^{2}=|e^{1-e^{-2it}}|=|e^{1-\cos 2t}e^{i\sin 2t}|  \le e^{ct^{2}+O(t^{4})}$ and
 \[\EE[G_t(X)]=(1-2\ln w_t)^{-\frac{\ell}{2}}=z^\ell.\]
\item \emph{Negative Binomial:} Let $w_t=1/p-(1/p-1)e^{-it}$ then $ |w_t|^{2}=\frac{1+(1-p)^{2}-2(1-p)\cos t}{p^{2}}  = 
\frac{p^2+4(1-p)\sin^2(t/2)}{p^2}\le e^{\frac{(1-p)t^2}{p^2}}$ and
\[\EE[G_t(X)]= \Big(\frac{1-p}{1-pw_t }\Big)^{r}=z^r.\]
\end{itemize}

\subsection{Additional calculations for Theorem~\ref{thm:tv}.}

\begin{itemize} 
\item \emph{Gaussian:}
The characteristic function of a Gaussian $X \sim \mathcal{N}(\mu,\sigma^2)$ is 
\begin{align*}
C_\Nc(t)=\EE e^{it X}=e^{it\mu-\frac{t^{2}\sigma^{2}}{2}}.
\end{align*}   
Therefore we have that 
\begin{align*}
C_\Mc(t)-C_{\Mc'}(t) \ge \frac{e^{-\frac{t^{2}\sigma^{2}}{2}}}{k}    \sum_{j =1}^k (e^{it \mu_j}-  e^{it \mu_j'}) .
\end{align*}
Now, using Lemma~\ref{lem:npb}, there exist an absolute constant $c$ such that,  
\begin{align*}
\max_{-\frac{\pi}{\epsilon L}\le t \le \frac{\pi}{\epsilon L}} \big| \sum_{j =1}^k (e^{it \mu_j}-  e^{it \mu_j'})\big| \ge e^{-cL}.
\end{align*}
Also, for $t \in (-\frac{\pi}{\epsilon L}, \frac{\pi}{\epsilon L}),$ $e^{-\frac{t^{2} \sigma^{2}}{2}} \ge e^{-\frac{\sigma^2\pi^2}{2\epsilon^2 L^2}}.$ And therefore,
\begin{align*}
\Big|C_\Mc(t)-C_{\Mc'}(t)\Big| \ge \frac{1}{k}  e^{-\frac{\sigma^2\pi^2}{2\epsilon^2 L^2}-cL}.
\end{align*}
By substituting $L = \frac{(\pi\sigma)^{2/3}}{(\epsilon^2c)^{1/3}}$ above we conclude that there exists $t$ such that 
\begin{align*}
\Big|C_\Mc(t)-C_{\Mc'}(t)\Big| \ge \frac{1}{k}  e^{-\frac32(c\pi\sigma/\epsilon)^{2/3}}.
\end{align*}
Now using 
 Lemma~\ref{lem:chartv}, we have
$\variation{\mathcal{M}' -\mathcal{M}} \geq k^{-1} \exp(-\Omega((\sigma/\epsilon)^{2/3}))$.

\item \emph{Chi-Squared:}
Let $X \sim \Mc$ and $X' \sim \Mc'$. Then, for $w=\exp(1/2-e^{-2it}/2)$, from Lemma~\ref{lem:list},
$$
\EE(w^X) - \EE(w^{X'}) = \frac1k\sum_{j=1}^k (e^{it \ell_j} - e^{it \ell'_j}) .
$$
Now we use  
 Lemma~\ref{lem:chartv}, with $\Omega' = [0, 2N]$ we have,
  \begin{align*}
\|\Mc -\Mc'\|_{TV} \ge   e^{-2ct^2 N} \Big(\left| \EE(w^X) - \EE(w^{X'}) \right|- \int_{x >2N}\exp(ct^2x) f(x) dx\Big),
  \end{align*}
  where $f\sim \chi^2(N)$.
 We have,
\begin{align*}
 \int_{x >2N}\exp(ct^2x) f(x) dx& = \frac{1}{(1-2ct^2)^{N/2-1}}\int_{y >2N(1-2ct^2)}f(y)dy \le \frac{e^{-N(1-4ct^2)^2/8}}{(1-2ct^2)^{N/2-1}}\\
 & \le \exp(-\Omega(N)),
\end{align*}
where we have used the pdf of chi-squared distribution and the tail bounds for chi-squared.
Now using Lemma~\ref{lem:npb}, and taking $|t| \le \frac{\pi}{L}$,
$$
 \|\Mc -\Mc'\|_{TV} \ge  k^{-1} e^{-c'L-2ct^2N} - \exp(-\Omega(n)) \ge  k^{-1}\exp( -c'L-2\pi^2N/L^2)- \exp(-\Omega(N)).
$$
Again setting, $L = N^{1/3}$,
$$
 \|\Mc -\Mc'\|_{TV} \ge  k^{-1} \exp(- \Omega(N^{1/3})).
$$
 
\item \emph{Negative-Binomial:} 
Let $X \sim \Mc$ and $X' \sim \Mc'$. Then, for $w=1/p - (1/p - 1)e^{-it}$, from Lemma~\ref{lem:list}, taking $G(x) = w^x$,
$$
\EE(w^X) - \EE(w^{X'}) = \frac1k\sum_{j=1}^k (e^{it r_j} - e^{it r'_j}) .
$$
Now we use  
 Lemma~\ref{lem:chartv}, with $\Omega' = [0, 6pN/(1-p)]$ we have,
  \begin{align*}
\|\Mc -\Mc'\|_{TV} \ge   e^{-12ct^2N/p} \Big(\left| \EE(w^X) - \EE(w^{X'}) \right|- \sum_{x >\frac{6Np}{1-p}}|w|^xu(x)\Big),
  \end{align*}
  where $u(x) =  \binom{x+N-1}{x}(1-p)^Np^x$. We have  $|w| \le e^{c(1-p)t^2/p^2} \le e^{c(1-p)/p^2}$ for $t<1$.
 Using Lemma~\ref{lem:tail}, with $X\sim NB(N,p)$, we have,
\begin{align*}
\sum_{x >\frac{6Np}{1-p}}\exp(cx(1-p)/p^2)u(x)& \le a^{1-\frac{6Np}{1-p}}\EE[a^{2X}] = a^{1-\frac{6Np}{1-p}}\Big(\frac{1-p}{1-pa^2}\Big)^N = \exp(-\Omega(N)),
\end{align*}
where, $ a = \exp(c(1-p)/p^2)>1$.
Now using Lemma~\ref{lem:npb}, and taking $|t| \le \frac{\pi}{L}$,
\begin{align*}
 \|\Mc -\Mc'\|_{TV}& \ge  k^{-1} e^{-c'L-12ct^2N/p} - \exp(-\Omega(n)) \\
 &\ge  k^{-1}\exp( -c'L-12\pi^2N/(pL^2))- \exp(-\Omega(N)).
\end{align*}
 Setting $L = (N/p)^{1/3}$,
$$
 \|\Mc -\Mc'\|_{TV} \ge  k^{-1} \exp(- \Omega((N/p)^{1/3})).
$$
\end{itemize}

\subsection{Proof of Theorem \ref{thm:multiplicity}}

%
Let $\bf{a}$ be the characteristic vector of a subset $S \subset \mathcal U$. Let $s_\ell=m_\ell(S)$ on this  set and let ${\bf s}=(s_0, s_1,\ldots, s_{k-1})$. We need to prove $\bf a$ is uniquely determined by $\bf s$.

Let us define 
\begin{align*}
n_{i,p}(\mathbf{a}) :=\sum_{r \equiv_p i} a_r \pmod{p} \ .
\end{align*} 

\begin{claim} For a prime number $p$ and  $i \not\equiv_p 0$, we have
$
n_{i,p}(\mathbf{a}) \equiv_p s_{0}-\sum_j {p-1 \choose j} s_{j} (-i)^{p-1-j} 
$\end{claim}

\begin{proof}
\begin{align*}
n_{i,p}(\mathbf{a})=\sum_{r\equiv_p i} a_{r} \pmod{p}
\end{align*}
Recall that Fermat's theorem (\cite{hardy1979introduction}) says that for any prime $p$ and any number $\alpha \not \equiv_p 0$, we must have that $\alpha^{p-1} \equiv_p 1$. Hence, for a prime number $p$ and some number $i \not \equiv_p 0$, we have 
\begin{align*}
s_{0}(\mathbf{a)}-\sum_{j} {p-1 \choose j} s_j (-i)^{p-1-j} 
&\equiv_p \sum_{r} a_r -\sum_{j} {p-1 \choose j} \sum_{r} a_r r^{j} (-i)^{p-1-j} \\
&\equiv_p  \sum_{r} a_r -\sum_{r} a_r  \sum_{j} {p-1 \choose j} r^{j} (-i)^{p-1-j} \\
&\equiv_p  \sum_{r} a_r -\sum_r a_r (r-i)^{p-1} \\
&\equiv_p \sum_{r \equiv_p i} a_r \equiv_p  n_{i,p} (\mathbf{a}) \ .
\end{align*}

\end{proof}

Since the value of $n_{i,p}$ is at most $\lceil {qn}/{p} \rceil$,  we can obtain the value of $n_{i,p}$ exactly if $p$ is chosen to be greater than $\sqrt{qn}$. Now, let us denote the vector $\mathbf{v}_{i,p} \in \mathbb{F}_{q}^{n}$ where the $\ell $th entry is
\begin{align*}
v_{i,p}[\ell]=
\begin{cases}
1 & \text{ if } \ell \equiv_p i \\
0 & \text{  otherwise } 
\end{cases} \ .
\end{align*} 
Therefore, consider two different subsets $S,S' \subset \mathcal{U}$ and assume that their characteristic vectors are $\bf{a}$ and $\bf{b}$ respectively. Therefore, if $\bf{a}$ and $\bf{b}$ both give rise to the same value of $\bf{s}$, then $\mathbf{a}.\mathbf{v}_{i,p}=\mathbf{b}.\mathbf{v}_{i,p}$. Hence, if the set of vectors 
\begin{align*}
\mathcal{S}=\{v_{i,p} \mid \sqrt{qn}\le p \le k,  0 \le i \le p-1, p \mbox{ prime} \}
\end{align*}
spans $\mathbb{F}_{q}^{n}$, then it must imply that $\mathbf{a}=\mathbf{b}$ and our proof will be complete. Consider a subset $\mathcal{T} \subset \mathcal{S}$ defined by
\begin{align*}
\mathcal{T}=\{v_{i,p} \mid \sqrt{qn}\le p \le k, 1 \le i \le p-1, p \mbox{ prime} \}
\end{align*}
Now, there are two possible cases. First, let us assume that the vectors in $\mathcal{T}$ are not all linearly independent in $\mathbb{F}_q$. In that case, we must have a set of tuples $(i_1,p_1),(i_2,p_2),\dots,(i_m,p_m)$ such that  
\begin{align}
\sum_{j=1}^{m} \alpha_j \mathbf{v}_{(i_j,p_j)} \equiv_q 0 \label{eq:assumption}
\end{align}  
where $0\neq \alpha_j \in \mathbb{F}_q $ for all $j$.
Now, by the Chinese Remainder Theorem, we can find an integer $r$ such that $r \equiv_{p_1} i_1$ and $r \equiv_{p_j} 0$ for all $p_j \neq p_1$. Define an infinite dimensional vector $\tilde{\mathbf{v}}$ where the $\ell$th entry is 
\begin{align*}
\tilde{\mathbf{v}}[\ell] =\sum_{j=1}^{m} \alpha_j \mathds{1}\Big[\ell \equiv_{p_j} i_j \Big]
\end{align*} 
Since, $i_j \not \equiv_{p_j} 0$, it is evident that $\tilde{\mathbf{v}}[r] \not \equiv_q 0$
Now, let $s$ be the smallest number such that $\tilde{\mathbf{v}}[s] \neq 0$ and $s>n$ because of our assumption in Eq.~\ref{eq:assumption}. Now consider the vector $\mathbf{v}_{t}$ where 
\begin{align*}
\mathbf{v}_t=\sum_{j=1}^{m}\alpha_j \mathbf{v}_{i_j-s+t,p_j}
\end{align*}
Now, $\mathbf{v}_{t}^{i}=0$ for all $i<t $ and $\mathbf{v}_{t}^{t} \neq 0$. Hence, the set $\{ \mathbf{v}_t\}_{t=1}^{n}$ are in the span of $\mathcal{S}$ and also span $\mathbb{F}_{q}^{n}$. 

For the second case, let us assume that the vectors in $\mathcal{T}$ are linearly independent. We require the size of $\mathcal{T}>n$ so that the vectors in $\mathcal{T}$ span $\mathbb{F}_{q}^{n}$. From the prime number theorem we know that 
\begin{align*}
\sum_{p~\mathrm{prime}:p <x}p  \sim \frac{x^{2}}{2 \log x}
\end{align*}
and hence we simply need that
\begin{align*}
\frac{k^{2}}{2\log k}-\frac{qn}{\log n} >n \ . 
\end{align*}
Therefore, $k>(1+o(1))\sqrt{qn \log qn}$ is sufficient.


\subsection{Algebraic method for Geometric distribution}
We will denote the Geometric distribution with success parameter $0<p<1$ as $\textup{Geo}(p)$ and it has the following form: for a random variable $X$ distributed according to $\textup{Geo}(p)$, $\Pr(X=x)=(1-p)^{x}p$ where $x \in \{0,1,2,\dots\}$.

\begin{theorem}[Learning mixtures of Geometric Distribution]
Let $\mathcal{M}=\frac{1}{k}\sum_{i=1}^k\textrm{Geo}(p_i)$
be a uniform mixture of $k$ Geometric distributions, with unknown probabilities 
\[p_1, \ldots, p_k \in \{\frac{1}{1+n\epsilon},\frac{1}{1+(n-1)\epsilon}, \ldots, 1\} \ .\] Then, the first $4\sqrt{n}$ moments suffice to learn the parameters $p_i$ and  there exists an algorithm that, when given
$O\Big(k^2 \Big(\frac{\sqrt{n}}{\epsilon}\Big)^{8\sqrt{n}}\Big)$ samples from
$\mathcal{M}$, exactly identifies the parameters $\{p_i\}_{i=1}^k$
with high probability.
\end{theorem}

\paragraph{Computing the moments.}
We compute the $\ell$th moment in the natural way again. Let $Y_1, \ldots, Y_t\sim X$ and let \[S_\ell=\sum Y_i^\ell /t \ . \]

\begin{lemma}[Restating Lemma \ref{lem:boundmoments}]
$\Pr[|S_\ell-\avg X^\ell|\geq \gamma] \leq \frac{\avg X^{2\ell}}{t\gamma^2} \leq \frac{(2\ell)! } {\gamma^2 t} \inf_\alpha \left ( \frac{\avg e^{\alpha X}}{\alpha^{2\ell}} \right )$ where the last inequality assumes the all the moments of $X$ are non-negative. 
\end{lemma}

The following corollary, tailors the above lemma for a mixture of geometric distributions. 
\begin{corollary}\label{bincorol}
If $X\sim \sum_{i=1}^k \textup{Geo}(p_i)/k$ then $\Pr[|S_\ell-\avg X^\ell |\geq \gamma] \le \frac{2}{t\gamma^2}\Big(\frac{4\ell}{\min_i p_i}\Big)^{2\ell+1}$. 
\end{corollary}
\begin{proof}
Given a random variable $Z\sim \textup{Geo}(p)$, we will show that $\EE Z^{k} \le 2\Big(\frac{2k}{p}\Big)^{k+1}$ for all integer valued $k \ge 0$. It is known that~\citep{Geom}
\begin{align*}
\EE Z^{k}=p\textup{Li}_{-k}(1-p)
\end{align*}
where $\textup{Li}_{-k}(z)$ is the polylogarithmic function of order $-k$ and argument $z$, defined explicitly as
\begin{align*}
\textup{Li}_{-k}(1-p)=\frac{1}{p^{k+1}}\sum_{j=0}^{k-1} \genfrac<>{0pt}{}{k}{j} (1-p)^{k-j}
\end{align*}
with $\genfrac<>{0pt}{}{k}{j}$ being the Eulerian numbers (see below). Hence, it can be observed that $\EE Z^k$ is a polynomial in $\frac{1}{p}$ of degree $k$. Denoting $C_{k}=\max_{0 \le j\le k-1}\genfrac<>{0pt}{}{k}{j}$ and substituting it, we get that 
\begin{align*}
\EE Z^k \le \frac{C_{k}}{p^k}\sum_{j=0}^{k-1} (1-p)^{k-j}=C_{k}\Big(\frac{1}{p}-1\Big)\Big(\frac{1}{p^k}-\Big(\frac{1}{p}-1\Big)^{k}\Big) < \frac{2C_k}{p^{k+1}}.
\end{align*}
From the definition of Eulerian numbers, we can also see that
\begin{align*}
\genfrac<>{0pt}{}{k}{j}=\sum_{t=0}^{j}(-1)^{t}{k+1 \choose t} (j+1-t)^{k} \le (j+1)^{k}2^{k+1} < (2k)^{k+1}.
\end{align*}
Putting everything together and by appealing to Lemma \ref{lem:boundmoments}, we get the statement of the corollary.
\end{proof}

For the geometric distribution, 
\[\avg X^\ell=\sum_{i=1}^k f(1/p_i)/k\] where $f$ is a degree $\ell$ polynomial with integer coefficients. If $1/p_i-1$ is an integer multiple of $\epsilon$ then this implies $k(\avg X^\ell)/\epsilon^\ell$ is integral and therefore any mixture with a different $\ell$th moment must differ by at least $\epsilon^\ell/k$. Hence, learning the $\ell$th moment up to $\gamma_\ell<\epsilon^\ell/(2k)$ implies learning the moment exactly. 

\begin{lemma}
$O\Big(k^2 \Big(\frac{\sqrt{n}}{\epsilon}\Big)^{8\sqrt{n}}\Big)$ samples are sufficient to exactly learn the first $4\sqrt{n}$ moments of a uniform mixture of $k$ Geometric distributions $\sum_{i=1}^k \textup{Geo}(p_i)/k$ with probability at least $7/8$ where each $\frac{1}{p_i} \in \{1,1+\epsilon, 1+2\epsilon, \ldots, 1+n\epsilon\}$.
\end{lemma}
\begin{proof}
Let $T=4\sqrt{n}$. From Corollary \ref{bincorol} and the preceding discussion, learning the $\ell $th moment exactly with failure probability $1/9^{1+T-\ell}$ requires 
\[
t= \gamma_\ell^{-2} 2\Big(4 \ell\Big)^{2\ell+1} 9^{1+T-\ell} = O\Big(k^2 9^{1+T-\ell} {\ell}^{2\ell} /\epsilon^{2\ell}\Big)= O\Big(k^2 9^{T} \Big(\frac{\ell}{3\epsilon}\Big)^{2\ell}\Big)
\]
samples. And hence, we can compute all $\ell$th moments exactly for $1\leq \ell\leq 4\sqrt{n}$ using \[\sum_{\ell=1}^T O\Big(k^2 9^{T} \Big(\frac{\ell}{3\epsilon}\Big)^{2\ell}\Big)=O\Big(k^2 \Big(\frac{T}{\epsilon}\Big)^{2T}\Big)\] samples with failure probability $\sum_{\ell=1}^T 1/9^{1+T-\ell}<\sum_{i=1}^\infty 1/9^{i}=1/8$.
\end{proof}

\paragraph{How many moments needed to determine  the parameters?}
It remains to show the first $4\sqrt{n}$ moments suffice to determine the $p_i$ values 
in the mixture $X\sim \sum_{i=1}^k \textup{Geo}(p_i)/k$. To do this suppose there exists another mixture $Y\sim \sum_{i=1}^k \textup{Geo}(q_i)/k$ and we will argue that 
\[
\avg X^\ell = \avg Y^\ell \mbox{ for } \ell=0, 1, \ldots, 4\sqrt{n} 
\]
implies $\{p_i\}_{i\in [k]}=\{q_i\}_{i\in [k]}$. To argue this, define integers $\alpha_i,\beta_i \in \{0,1,\ldots, n\}$ such that $p_i=\frac{1}{1+\alpha_i \epsilon}$ and $q_i=\frac{1}{1+\beta_i \epsilon}$. Let $\mathcal A=\{\alpha_1, \ldots, \alpha_k\}$ and $\mathcal B=\{\beta_1, \ldots, \beta_k\}$ . Then,
\[
\avg X = \avg Y   \Longrightarrow \sum_i 1/p_i=\sum_i 1/q_i \Longrightarrow  \sum_i \alpha_i=\sum_i \beta_i \Longrightarrow m_1(\mathcal A) = m_1(\mathcal B) \]
and, after some algebraic manipulation, it can be shown that for all $\ell\in \{2,3,\ldots\}$, 
\begin{align*}
\left (\forall \ell'\in \{0,1,\ldots, \ell-1\}~,~ \sum_i \alpha_i^{\ell'} =\sum_i \beta_i^{\ell'} \right )~\mbox{ and }~ \avg X^\ell = \avg Y^\ell
 &\Longrightarrow   \sum_i \alpha_i^\ell=\sum_i \beta_i^\ell \\
& \Longrightarrow m_\ell (\mathcal A) = m_\ell(\mathcal B)  \ .
\end{align*}
Hence, if the first $T$ moments match, $m_\ell(\mathcal A)= m_\ell(\mathcal B)$ for all $\ell=0,1,\ldots, T$. But, again Theorem \ref{kras} establishes that if $T=4\sqrt{n}$ then this implies $\mathcal A=\mathcal B$.

\paragraph{Alternative Technique.}
In the previous analysis the parameters of the geometric distribution ($p_i$'s) had to belong to the set $\{1,\frac{1}{1+\epsilon},\frac{1}{1+2\epsilon},\dots,\frac{1}{1+n\epsilon}\}$. The reason we had to choose this set is because the moments were polynomials in inverse of the parameters ($\frac{1}{p_i}$'s). However it is also possible to obtain a sample complexity bound when the parameters belong to the set $\{0,\epsilon,2\epsilon,\dots,1\}$. This can be done by estimating the probability mass function of the mixture at the discrete points $\{0,1,2,\dots\}$. We have the following theorem in this case.

\begin{theorem}[Learning mixtures of geometric distributions (alternative)]
Let $\mathcal{M}=\frac{1}{k}\sum_{i=1}^k\textrm{Geo}(p_i)$
be a uniform mixture of $k$ geometric distributions, with unknown probabilities $p_1, \ldots, p_k \in \{0,\epsilon,\dots,1\}$. Then, the first $4/\sqrt{\epsilon}$ moments suffice to learn the parameters $p_i$ and  there exists an algorithm that, when given
$O\Big(\frac{k^{2}}{\epsilon^{8/\sqrt{\epsilon}+2}}\log \frac{1}{{\epsilon}}\Big)$ samples from
$\mathcal{M}$, exactly identifies the parameters $\{p_i\}_{i=1}^k$
with high probability.
\end{theorem}

Recall that for a random variable $X \sim \mathcal{M}$ distributed according to the mixture of geometric distributions, we have
\begin{align*}
&\Pr(X=0)=\frac{1}{k}\sum_i p_i \\
&\Pr(X=1)=\frac{1}{k} \sum_i p_i-p_i^2 \\
&\Pr(X=2)=\frac{1}{k} \sum_i p_i-2p_i^2+p_i^3 
\end{align*}
and more generally,
\begin{align*}
\Pr(X=k)=\frac{1}{k} \sum_i (1-p_i)^{k}p_i
\end{align*}
which is a polynomial in degree $k+1$.
Now, for the mixture $X\sim 1/k \sum_{i=1}^k \textup{Geo}(p_i)$, we need to argue that estimating the probabilities $\Pr(X=\ell)$ for $\ell=0,1,\dots,4\sqrt{1/\epsilon}$ is sufficient to recover the parameters $p_i$. Again, suppose there exists another mixture $Y\sim 1/k\sum_{i=1}^k \textup{Geo}(q_i)$ such that 
\[
\Pr(X=\ell) = \Pr(Y=\ell) \mbox{ for } \ell=0, 1, \ldots, 4\sqrt{1/\epsilon} 
\]
and we will argue that this implies $\{p_i\}_{i\in [k]}=\{q_i\}_{i\in [k]}$. As before, define integers $\alpha_i,\beta_i \in \{0,1,\dots,\frac{1}{\epsilon}\}$ such that $p_i=\alpha_i\epsilon$ and $q_i=\beta_i\epsilon$. 
Let $\mathcal A=\{\alpha_1, \ldots, \alpha_k\}$ and $\mathcal B=\{\beta_1, \ldots, \beta_k\}$ and it can be shown after some algebraic manipulations that 
\begin{align*}
 \Big(\forall \ell'\in \{0,1,\ldots, \ell-1\}~,~ \sum_i \alpha_i^{\ell'+1} &=\sum_i \beta_i^{\ell'+1} \Big) ~\mbox{ and }~ \Pr(X=\ell) = \Pr(Y=\ell) \\
&\implies   \sum_i \alpha_i^{\ell+1}=\sum_i \beta_i^{\ell+1}  \implies m_{\ell+1} (\mathcal A) = m_{\ell+1}(\mathcal B)   .
\end{align*}
Notice that $m_{0}(\mathcal{A})=m_{0}(\mathcal{B})$ trivially because both of them contain $k$ components. Again, Theorem \ref{kras} establishes that if $m_{\ell} (\mathcal A) = m_{\ell}(\mathcal B)$ for $\ell \in \{0,1,\dots 4\sqrt{1/\epsilon}\}$ then this implies $\mathcal A=\mathcal B$.

\paragraph{Computing the probabilities.} 
Suppose $Y_1,Y_2,\dots,Y_t$ are i.i.d. with  $X\sim 1/k \sum_{i=1}^k \textup{Geo}(p_i)$. Let us denote $S_{\ell}$ as the empirical probability that we calculate as,
\begin{align*}
S_{\ell}=\frac{\sum_{i=1}^t \mathds{1}[Y_i=\ell]}{t}.
\end{align*}
It is obvious that $\avg S_{\ell}=\Pr(X=\ell)$. 
Now, using Chernoff bound, we have
\begin{align*}
\Pr(|S_{\ell}-\Pr(X=\ell)| \ge \gamma_{\ell}) \le 2e^{-t\gamma_{\ell}^{2}/3}.
\end{align*}
Again, recall that 
\begin{align*}
\Pr(X=\ell)=\sum_{i} \frac{f(p_i)}{k}
\end{align*}
where $f(\cdot)$ is a polynomial of degree $\ell+1$ with integer coefficients. If $p_i$ is an integer multiple of $\epsilon$ then this implies $kS_{\ell}/\epsilon^{\ell+1}$ is integral and therefore any mixture with a different $\ell$th moment has a $\ell$ moment that differs by at least $\epsilon^{\ell+1}/k$. Hence, learning the $\ell$th moment up to $\gamma_\ell<\epsilon^{\ell+1}/(2k)$ implies learning the moment exactly. 
We will use $t=\frac{12k^{2}}{\epsilon^{8/\sqrt{\epsilon}+2}}\log \frac{64}{\sqrt{\epsilon}}$ number of samples and we will show it will be sufficient to succeed with a probability of at least $\frac{7}{8}$. We will estimate the probabilities as mentioned above and therefore the failure probability can be calculated by using the Chernoff Bound and a union bound over $\frac{4}{\sqrt{\epsilon}}$ probabilities to be estimated. Therefore the probability of failure is bounded above by,
\begin{align*}
 \sum_{\ell} 2\exp(-t \gamma_{\ell}^{2}/3) \le \frac{8}{\sqrt{\epsilon}} \max_{\ell} \exp(-t \gamma_{\ell}^{2}/3)&=\frac{8}{\sqrt{\epsilon}} \exp(-t \min_{\ell} \gamma_{\ell}^{2}/3) \\
&= \frac{8}{\sqrt{\epsilon}}\exp(-t \epsilon^{8/\sqrt{\epsilon}+2}/(12k^{2}))\le  \frac{1}{8}
\end{align*}
and hence the proof is complete.

\subsection{Proof of Lemma \ref{lem:ind}}
We will prove that for $X\sim \textup{Bin}(n,p)$, the leading term of $\avg X^{\ell}$ is  $\prod_{i=0}^{\ell-1} (n-i)p^{\ell}$. Since for $n \geq \ell$, $\prod_{i=0}^{\ell-1} (n-i) \neq 0$, this implies that $\avg X^{\ell}$ is a polynomial of degree exactly $\ell$.
We will prove this by induction. Since $X\sim \textup{Bin}(n,p)$, we know that  
$
\avg X= np.
$
This verifies the base case. Now, in the induction step, let us assume that the leading term of $\avg X^{k}$ is $\prod_{i=0}^{k-1} (n-i)p^{k}$. It is known that (see \cite{belkin2010polynomial}) 
\begin{align*}
\avg X^{k+1} =np \avg X^k+p(1-p)\frac{d \avg X^{k}}{dp}.  
\end{align*}
Therefore it follows that the leading term of $\avg X^{k+1}$ is  
\begin{align*}
np\prod_{i=0}^{k-1} (n-i)p^{k}-kp^2\prod_{i=0}^{k-1} (n-i)p^{k-1}=\prod_{i=0}^{k} (n-i)p^{k+1}.
\end{align*}
This proves the induction step and the lemma.

\end{document}